\documentclass[journal]{IEEEtran}
\usepackage{amssymb}
\usepackage{amsmath}
\usepackage{amsfonts}
\usepackage{graphicx}

\newtheorem{theorem}{Theorem}
\newtheorem{lemma}{Lemma}

\newtheorem{remark}{Remark}
\newenvironment{proof}[1][]{Proof}{}
\begin{document}

\author{Milad Khaledyan$^{1}$,~\IEEEmembership{Member,~IEEE,} Tairan
Liu$^{2}$, and Marcio de Queiroz$^{2}$,~%
\IEEEmembership{Senior Member,~IEEE\thanks{$^{1}$M. Khaledyan is with the Department of Electrical \& Computer Engineering, University of New Mexico, Albuquerque, NM 87131 USA (milad@unm.edu).}\thanks{$^{2}$T. Liu and M. de Queiroz are with the Department of Mechanical \& Industrial
        		Engineering, Louisiana State University, Baton Rouge, LA 70803 USA (tliu7@lsu.edu, mdeque1@lsu.edu).}}
}
\title{Flocking and Target Interception Control for Formations of
Nonholonomic Kinematic Agents}
\maketitle

\begin{abstract}
In this work, we present solutions to the flocking and target interception
problems of multiple nonholonomic unicycle-type robots using the
distance-based framework. The control laws are designed at the kinematic
level and are based on the rigidity properties of the graph modeling the
sensing/communication interactions among the robots. An input transformation
is used to facilitate the control design by converting the nonholonomic
model into the single integrator-like equation. We assume only a subset of
the robots know the desired, time-varying flocking velocity or the target's
motion. The resulting control schemes include distributed, variable
structure observers to estimate the unknown signals. Our stability analyses
prove convergence to the desired formation while tracking the flocking
velocity or the target motion. The results are supported by experiments.
\end{abstract}


\IEEEpeerreviewmaketitle

\section{Introduction}

The field of decentralized control of multi-agent systems is an ongoing
topic of interest to control and robotics researchers. Formation control is
a type of coordinated behavior where mobile agents are required to
autonomously converge to a specified spatial pattern. Many
coordinated/cooperative tasks, such are element tracking, exploration, and
object transportation, also require the formation to maneuver as a virtual
rigid body. Such maneuvers can include translation, rotation, or the
combination of both. When only the translational component is considered,
the problem is often referred to as flocking. A related problem is called
target interception where the agents intercept and surround a moving target
with a given formation.

Formation control algorithms have been designed for different models of the
agent motion. Most results are based on point-mass type models, such as the
single and double integrator models. For example, see \cite%
{Dorfler,Krick,ZhangJDSMC} for single integrator results and \cite%
{CaiJDSMC,Cao,Sun 2017} for double integrator results. On the other hand,
some results have used more sophisticated models that account for the agent
kinematics/dynamics. One of two models are used in these cases: the
fully-actuated (holonomic) Euler-Lagrange model, which includes robot
manipulators, spacecraft, and some omnidirectional mobile robots; or the
nonholonomic (underactuated) model, which accounts for velocity constraints
that typically occur in the vehicle motion (e.g., differentially-driven
wheeled mobile robots and air vehicles). In the nonholonomic case, models
can be further subdivided into two categories: the purely kinematic model
where the control inputs are at the velocity level, and the dynamic model
where the inputs are at the actuator level. Examples of work based on the
Euler-Lagrange model include \cite{CaiTCST,Chen,Chung,Lee,Pereira}.
Formation control results based on nonholonomic kinematic models can be
found in \cite{Baillieul03,Mastellone,Moshtagh,Sadowska}. Designs for
nonholonomic dynamic models appeared in \cite%
{ChenIJRR,Dong08,Dong09,GaziJDSMC,Liang}.

Flocking and target interception controllers were introduced in \cite%
{CaiAJC,CaiACC14} for the single- and double-integrator models using the
distance-based, rigid graph approach from \cite{Krick} where the
time-varying flocking velocity was available to all agents. A 2D formation
maneuvering controller was proposed in \cite{Bai} for the double-integrator
model where the group leader, who has inertial frame information, passes the
information to other agents through a directed path in the graph. A
limitation of this control is that it becomes unbounded if the desired
formation maneuvering velocity is zero. In \cite{KhaledyanR}, a
leader-follower type solution modeled as a spanning tree was presented for
the formation maneuvering problem based on the nonholonomic kinematics of
unicycle robots. A combination of tracking errors and inter-agent
coordination errors were used to quantify the control objective. A consensus
scheme was presented in \cite{Han} using both the single- and
double-integrator models where the desired flocking velocity is constant and
known to only two leader agents. In \cite{Rozenheck}, the flocking strategy
involved a leader with a constant velocity command and followers who track
the leader while maintaining the formation shape. The control law, which was
based on the single-integrator model, consisted of the standard gradient
descent formation acquisition term plus an integral term to ensure zero
steady-state error with respect to the velocity command. In \cite{Sun 2016},
a flocking controller was designed for agents modeled by double integrators
that allows all agents to both achieve the same velocity and reach a desired
formation in finite time. A similar problem was addressed in \cite{Deghat}
but with asymptotic formation acquisition and velocity consensus. In \cite%
{Marina}, a controller was proposed using the single-integrator model that
can steer the entire formation in rotation and/or translation in 3D. The
rotation component was specified relative to a body-fixed frame whose origin
is at the centroid of the desired formation and needs to be known. In \cite%
{Shi}, the authors study the flocking behavior of multiple vehicles with a
dynamic leader with known acceleration available to all agents. A flocking
controller was designed in \cite{Sun 2016} for double integrator agents that
ensures finite time convergence for the flocking of desired formation with
flocking velocity equal to the average of the agents' initial velocity.
Recently, \cite{Yang} introduced a distance-based, flocking-type controller
where the formation centroid tracks a reference trajectory using a
finite-time centroid observer.

A special case of flocking, called consensus tracking, where the agents
simply have to track the motion of a leader without being in formation was
addressed in \cite{Cao TAC 12,Hong,Mei}. In \cite{Hong}, only the
single-integrator agents connected to the leader had access to its position
and a decentralized, linear observer was designed to estimate the leader's
time-varying velocity. However, exact tracking of the leader motion was only
assured when the leader acceleration was known by all agents. In \cite{Mei},
Euler-Lagrange agents with parametric uncertainty were considered in the
design of two consensus tracking algorithms. In the first design, the leader
velocity was constant and an adaptive controller combined with a
distributed, linear velocity observer were developed. The second design
assumed the leader velocity is time-varying which leads to the formulation
of a variable structure-type control law using one- and two-hop neighbor
information. In \cite{Cao TAC 12}, the authors studied single and double
integrator agents in fixed and switching network topologies with constant
and time-varying leader velocity. Distributed variable structure consensus
tracking controllers were designed without velocity (resp., acceleration)
measurements for the single (resp., double) integrator case.

A popular formation control approach is to use the inter-agent distances as
the controlled variables. This approach is intrinsically related to rigid
graph theory \cite{Anderson08} since the concept of graph rigidity naturally
ensures that the inter-agent distance constraints of the desired formation
are enforced. The distance-based control framework has been mostly applied
to the single and double integrator agent models. To the best of our
knowledge, the only exceptions are the results in \cite{DimaDistance,Sun
2018}. In \cite{DimaDistance}, the authors considered the nonholonomic
kinematic model in the design of a formation acquisition controller. The
work in \cite{Sun 2018} studied the circular formation control of
nonholonomic kinematic agents with fixed (but distinct) cruising speeds.

In this paper, we apply the distance-based approach to the flocking and
target interception of nonholonomic kinematic agents in the form of
unicycle-type vehicles. In the flocking problem, we assume the desired,
time-varying flocking velocity is known by only a subset of the agents. In
the target interception problem, only the leader agent has the target
information. We use an input transformation to convert the nonholonomic
multi-agent system into a single integrator-like system that includes a
multiplicative matrix dependent on the vehicle heading angle error. This
transformation enables us to use the gradient descent law from \cite{Krick}
for formation acquisition augmented with a flocking or target interception
term. The flocking term for each agent is a flocking velocity estimate
generated by a distributed, variable structure observer using only neighbor
information, which was inspired by the consensus algorithms in \cite{Cao TAC
12,Mei}. The target interception term for the followers is composed of two
estimates$-$one for the target velocity and one for the relative position of
the target to the leader$-$which are also updated by distributed, variable
structure observers. For both problems, the overall closed-loop system is
composed of multiple coupled nonlinear subsystems. Thus, the stability of
the proposed observer-controller system is analyzed using input-to-state
stability and interconnected system theory. Our analyses show that the error
dynamics are asymptotically stable at the origin for both problems, meaning
that the flocking and target interception objectives are successively met.
The main contribution of this paper is that it is the first to apply the
distance-based framework to nonholonomic kinematic agents for the flocking
and target interception problems. A preliminary version of this work
appeared in \cite{KhaledyanACC} where the flocking velocity was assumed
known to all agents.

\section{Background Material}

\label{Preliminaries}

An undirected graph $G$ is a pair $(V,E)$ where $V=\{1,2,...,n\}$ is the set
of nodes and $E\subset V\times V$ is the set of undirected edges that
connect two different nodes, i.e., if node pair $(i,j)\in E$ then so is $%
(j,i)$. We let $a\in \left\{ 1,\ldots ,n(n-1)/2\right\} $ denote the total
number of edges in $E$. \ The set of neighbors of node $i$ is denoted by 
\begin{equation}
\mathcal{N}_{i}(E)=\{j\in V\mid (i,j)\in E\}.  \label{Ni}
\end{equation}

Let $A=[a_{ij}]\in 
\mathbb{R}
^{n\times n}$ be the adjacency matrix defined such that $a_{ij}=1$ if $%
(i,j)\in E$ and $a_{ij}=0$ otherwise. Note that $a_{ij}=a_{ji}$. The
Laplacian matrix $L=[l_{ij}]\in 
\mathbb{R}
^{n\times n}$ associated with $A$ is defined such that $l_{ii}=\sum_{j=1,j%
\neq i}^{n}a_{ij}$ and $l_{ij}=-a_{ij}$ for $i\neq j$. Note that $L$ is
symmetric positive definite, and has a simple zero eigenvalue with an
associated eigenvector $1_{n}$ where $1_{n}$ is the $n\times 1$ vector of
ones \cite{Chung-graph}.

If $p_{i}\in 
\mathbb{R}
^{2}$ is the coordinate of node $i$, then a\textit{\ }framework $F$ is
defined as the pair $\left( G,p\right) $ where $p=\left[ p_{1},\ldots ,p_{2}%
\right] \in 
\mathbb{R}
^{2n}$. In the following, we assume all frameworks have \textit{generic }%
properties, i.e., the properties hold for almost all of the framework
representations. This is done to exclude certain degenerate configurations
such as frameworks that lie in a hyperplane (see \cite{Graver} for a
detailed study of generic frameworks).

Based on an arbitrary ordering of edges, the edge function\textit{\ }$\phi :$
$%
\mathbb{R}
^{2n}\rightarrow 
\mathbb{R}
^{a}$ is given by%
\begin{equation}
\phi (p)=\left[ ...,\left\Vert p_{i}-p_{j}\right\Vert ^{2},...\right] ,\text{
\ \ }(i,j)\in E  \label{edge function}
\end{equation}%
such that its $k$th component, $\left\Vert p_{i}-p_{j}\right\Vert ^{2}$,
relates to the $k$th edge of $E$ connecting the $i$th and $j$th nodes. The
rigidity matrix\textit{\ }$R:%
\mathbb{R}
^{2n}\rightarrow 
\mathbb{R}
^{a\times 2n}$ is given by 
\begin{equation}
R(p)=\frac{1}{2}\frac{\partial \phi (p)}{\partial p}  \label{R}
\end{equation}%
where rank$\left[ R(p)\right] \leq 2n-3$ \cite{Asimow79}. Notice that the $k$%
th row of $R$ has the form 
\begin{equation}
\left[ 0\text{ }...\text{ }0\text{ }\left( p_{i}-p_{j}\right) ^{\top }\text{ 
}0\text{ }...\text{ }0\text{ }\left( p_{j}-p_{i}\right) ^{\top }\text{ }0%
\text{ }\ldots \text{ }0\right]  \label{R row}
\end{equation}%
where $\left( p_{i}-p_{j}\right) ^{\top }$ is in columns $2i-1$ and $2i$, $%
\left( p_{j}-p_{i}\right) ^{\top }$ is in columns $2j-1$ and $2j$, and all
other elements are zero.

An isometry of $%
\mathbb{R}
^{2}$ is a bijective map $\mathcal{T}:\mathbb{\ 
\mathbb{R}
}^{2}\rightarrow \mathbb{\ 
\mathbb{R}
}^{2}$ satisfying \cite{Izmestiev} 
\begin{equation}
\left\Vert w-v\right\Vert =\left\Vert \mathcal{T}\left( w\right) -\mathcal{T}%
\left( v\right) \right\Vert ,\quad \forall w,v\in 
\mathbb{R}
^{2}.  \label{iso}
\end{equation}%
This map includes rotations and translations of the\textit{\ }vector $w-v$.
Two frameworks are said to be \textit{isomorphic} in $%
\mathbb{R}
^{2}$ if they are related by an isometry. In this paper, we will represent
the collection of all frameworks that are isomorphic to $F$ by Iso$\left(
F\right) $. It is important to point out that (\ref{edge function}) is
invariant under isomorphic motions of the framework.

Frameworks $\left( G,p\right) $ and $\left( G,\hat{p}\right) $ are
equivalent if $\phi (p)=\phi (\hat{p})$, and are congruent if $\left\Vert
p_{i}-p_{j}\right\Vert =\left\Vert \hat{p}_{i}-\hat{p}_{j}\right\Vert $, $%
\forall i,j\in V$ \cite{Jackson07}. The necessary and sufficient condition
for a generic framework $\left( G,p\right) $ to be infinitesimally rigid is
rank$\left[ R\left( p\right) \right] =2n-3$ \cite{Izmestiev}. An
infinitesimally rigid framework is minimally rigid\ if and only if $a=2n-3$ 
\cite{Anderson08}. If the infinitesimally rigid frameworks $(G,p)$ and $(G,%
\hat{p})$ are equivalent but not congruent, then they are referred to as 
\textit{ambiguous }\cite{Anderson08}. The notation Amb$\left( F\right) $
will be used to represent the collection of all frameworks that are
ambiguous to the infinitesimally rigid framework $F$. All frameworks in Amb$%
\left( F\right) $\ are also assumed to be infinitesimally rigid. According
to \cite{Anderson08} and Theorem 3 of \cite{Aspnes}, this assumption holds
almost everywhere.

\begin{lemma}
\label{LemmaR1n}\cite{CaiAJC} For any $x\in 
\mathbb{R}
^{2}$, $R(p)(1_{n}\otimes x)=0$ where $1_{n}$ is the $n\times 1$ vector of
ones.
\end{lemma}

\begin{theorem}
\label{Khalil_Thm}\cite{Khalil} Consider the system $\dot{x}=f\left(
x,u\right) $ where $x$ is the state, $u$ is the control input, and $f(x,u)$
is locally Lipschitz in $(x,u)$ in some neighborhood of $(x=0,u=0)$. Then,
the system is locally input-to-state stable (ISS) if and only if the
unforced system $\dot{x}=f(x,0)$ has a locally asymptotically stable
equilibrium point at the origin.
\end{theorem}

\begin{theorem}
\label{Marquez}\cite{Khalil} Consider the interconnected system 
\begin{equation}
\begin{array}{ll}
\Sigma _{1}\text{:} & \dot{x}=f(t,x,y) \\ 
\Sigma _{2}\text{:} & \dot{y}=g(t,y).%
\end{array}
\label{interconn}
\end{equation}%
if subsystem $\Sigma _{1}$ with input $y$ is ISS and $y=0$ is a uniformly
asymptotically stable equilibrium point of subsystem $\Sigma _{2}$, then $%
[x,y]=0$ is a uniformly asymptotically stable equilibrium point of the
interconnected system.
\end{theorem}

For any piecewise continuous signal $x:\mathbb{R}_{\geq 0}\rightarrow 
\mathbb{R}^{n}$, 
\begin{equation}
\left\Vert x\right\Vert _{\mathcal{L}_{\infty }}:=\sup\limits_{t\geq
0}\left\Vert x(t)\right\Vert .  \label{inf norm}
\end{equation}%
If $\left\Vert x\right\Vert _{\mathcal{L}_{\infty }}<\infty $ (the signal is
bounded for all time), we say that $x(t)\in \mathcal{L}_{\infty }$.

Finally, for any $x\in 
\mathbb{R}
^{n}$, sgn$\left( x\right) :=[sgn\left( x_{1}\right) ,\ldots ,$ $sgn\left(
x_{n}\right) ]$ where $sgn(\cdot )$ is the standard signum function: 
\begin{equation}
sgn\left( x_{i}\right) =\left\{ 
\begin{array}{ll}
1 & \text{if }x_{i}>0 \\ 
0 & \text{if }x_{i}=0 \\ 
-1 & \text{if }x_{i}<0.%
\end{array}%
\right.  \label{sgn}
\end{equation}

\section{System Model}

\label{System Model}

Consider a system of $n$ agents moving autonomously on the plane. Figure \ref%
{vehicle} depicts the $i$th agent, where the reference frame $\left\{
X_{0},Y_{0}\right\} $ is an inertial frame. The body reference frame $%
\left\{ X_{i},Y_{i}\right\} $ is attached to the $i$th vehicle with the $%
X_{i}$ axis aligned with its heading (longitudinal) direction, which is
given by angle $\theta _{i}$ and measured counterclockwise from the $X_{0}$
axis. Point $C_{i}$ denotes the $i$th vehicle's center of mass which is
assumed to coincide with its center of rotation.
\begin{figure}[!t]
\centering
\includegraphics[keepaspectratio,width=1.8733in,height=1.6214in]{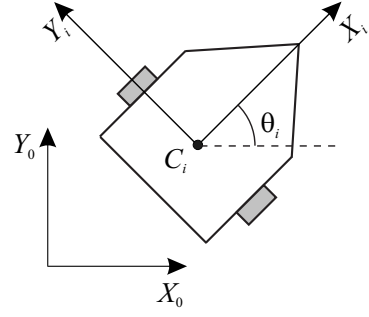}
\caption{Schematic of the unicycle agent.}
\label{vehicle}
\end{figure}
We assume the agent motion is
governed by the following nonholonomic, unicycle kinematic model 
\begin{equation}
\dot{q}_{i}=S(\theta _{i})\eta _{i},\quad i=1,...,n.  \label{dynamic1}
\end{equation}%
\newline
In (\ref{dynamic1}), $q_{i}=\left[ x_{i},y_{i},\theta _{i}\right] $ denotes
the position and orientation of $\{X_{i},Y_{i}\}$ relative to $%
\{X_{0},Y_{0}\}$, $\eta _{i}=\left[ v_{i},\omega _{i}\right] $ is the
control input, $v_{i}$ is the $i$th agent's translational speed in the
direction of $\theta _{i}$, $\omega _{i}$ is the $i$th agent's angular speed
about the vertical axis passing through $C_{i}$, and%
\begin{equation}
S(\theta _{i})=%
\begin{bmatrix}
\cos \theta _{i} & 0 \\ 
\sin \theta _{i} & 0 \\ 
0 & 1%
\end{bmatrix}%
.  \label{S}
\end{equation}

\section{Problem Statement}

\label{Problem Statement}

Consider that the agents' target formation is modeled by the framework $%
F^{\ast }=(G^{\ast },p^{\ast })$ where $G^{\ast }=\left( V^{\ast },E^{\ast
}\right) $, dim$(V^{\ast })=n$, dim$(E^{\ast })=a$, $p^{\ast }=\left[
p_{1}^{\ast },\ldots ,p_{n}^{\ast }\right] $, and $p_{i}^{\ast }=\left[
x_{i}^{\ast },y_{i}^{\ast }\right] $. The target distance separating the $i$%
th and $j$th agents is given by \ \ \ \ 
\begin{equation}
d_{ij}=\left\Vert p_{i}^{\ast }-p_{j}^{\ast }\right\Vert >0,\quad i,j\in
V^{\ast }.  \label{dij}
\end{equation}%
We assume $F^{\ast }$ is constructed to be infinitesimally and minimally
rigid.

The actual formation of the agents is encoded by the framework $F(t)=\left(
G^{\ast },p(t)\right) $ where $p=\left[ p_{1},\ldots ,p_{n}\right] $ and $%
p_{i}=\left[ x_{i},y_{i}\right] $. We make the following assumptions about
the agents.

\begin{enumerate}
\item[A1.] Agent $i$ can measure the relative position of agent $j$, $%
p_{i}-p_{j}$, $\forall j\in \mathcal{N}_{i}(E^{\ast })$ with respect to
frame $\left\{ X_{i},Y_{i}\right\} $.

\item[A2.] Agent $i$ can measure the relative heading angle of agent $j$, $%
\theta _{i}-\theta _{j}$, $\forall j\in \mathcal{N}_{i}(E^{\ast })$.%
\footnote{%
An alternative to A2 is to assume that each agent is equipped with a compass
to measure its own heading angle $\theta _{i}$.}

\item[A3.] Agent $i$ has a communication channel with agent $j$, $\forall
j\in \mathcal{N}_{i}(E^{\ast })$.
\end{enumerate}

We will address the following two formation problems.

\textit{Flocking Problem:} In this problem, the agents need to acquire and
maintain a pre-defined geometric shape in the plane while simultaneously
moving with a desired translational velocity that is known by only a subset
of the agents. That is, 
\begin{equation}
F(t)\rightarrow \text{Iso}\left( F^{\ast }\right) \text{ as }t\rightarrow
\infty ,  \label{control_obj_form}
\end{equation}%
which is equivalent to 
\begin{equation}
\left\Vert p_{i}(t)-p_{j}(t)\right\Vert \rightarrow d_{ij}\text{ as }%
t\rightarrow \infty ,\quad i,j\in V^{\ast }  \label{control_obj_form2}
\end{equation}%
due to the framework rigidity, and 
\begin{equation}
\dot{p}_{i}(t)-v_{0}(t)\rightarrow 0\text{ as }t\rightarrow \infty ,\quad
i=1,...,n  \label{2nd_Obj_Maintenance}
\end{equation}%
where $v_{0}\in 
\mathbb{R}
^{2}$ is any continuously differentiable function of time representing the
desired flocking velocity. We assume $v_{0}(t),\dot{v}_{0}(t)\in \mathcal{L}%
_{\infty }$ where $\left\Vert \dot{v}_{0}(t)\right\Vert _{\mathcal{L}%
_{\infty }}\leq \gamma _{0}$ and $\gamma _{0}$ is a known positive constant.
The nonempty subset of agents that have direct access to $v_{0}$ is denoted
by $V_{0}\subset V^{\ast }$.

\textit{Target Interception Problem}: Here, the agents should intercept and
enclose a (possibly evading) moving target with a pre-defined formation. The
acquisition of the pre-defined formation is quantified by (\ref%
{control_obj_form}). Let $p_{T}\in 
\mathbb{R}
^{2}$ denote the target position, which is assumed to be a twice
continuously differentiable function of time such that $p_{T}(t),\dot{p}%
_{T}(t),\ddot{p}_{T}(t)\in \mathcal{L}_{\infty }$ where $\left\Vert \ddot{p}%
_{T}(t)\right\Vert _{\mathcal{L}_{\infty }}\leq \gamma _{T1}$ and $\gamma
_{T1}\ $is a known positive constant. In order to intercept the target, we
use a leader-follower-like scheme where the $n$th agent is the leader and
the remaining agents are followers. The leader is responsible for tracking
the target while the followers flock and maintain the desired formation.
Thus, the leader is the only agent that can directly measure its relative
position to the target, $p_{T}-p_{n}$, and the target velocity, $\dot{p}_{T}$%
. The desired formation $F^{\ast }$ should be selected with the additional
condition that $p_{n}^{\ast }\in $ conv$\left\{ p_{1}^{\ast
},...,p_{n-1}^{\ast }\right\} $ where conv$\left\{ \cdot \right\} $ denotes
the convex hull. The main objective for this problem is that $p_{T}(t)$
approach conv$\left\{ p_{1}(t),...,p_{n-1}(t)\right\} $ as time evolves,
i.e.,%
\begin{equation}
p_{T}(t)\in \text{conv}\{p_{1}(t),\text{ }\ldots ,\text{ }p_{n-1}(t)\}\text{
as }t\rightarrow \infty .  \label{2nd_obj_target}
\end{equation}

\section{Flocking Control}

\label{KLC}

We begin by introducing several error variables. The relative position of
agents $i$ and $j$ is defined as 
\begin{equation}
{p}_{ij}=p_{i}-p_{j}  \label{qtilda}
\end{equation}%
while the corresponding distance error is captured by the variable \cite%
{Krick} 
\begin{equation}
z_{ij}=\left\Vert {p}_{ij}\right\Vert ^{2}-d_{ij}^{2}.  \label{zij}
\end{equation}%
The vector of all $z_{ij}$ for which $\left( i,j\right) \in E^{\ast }$ is
defined as $z=\left[ ...,z_{ij},...\right] \in 
\mathbb{R}
^{a}$, which is ordered as (\ref{edge function}). Given that $\left\Vert {p}%
_{ij}\right\Vert \geq 0$, note that $z_{ij}=0$ if and only if $\left\Vert {p}%
_{ij}\right\Vert =d_{ij}$. This means that when $z=0$, the frameworks $F$
and $F^{\ast }$ are equivalent and therefore, $F=$ Iso$\left( F^{\ast
}\right) $ or $F=$ Amb$\left( F^{\ast }\right) $. Next, let 
\begin{equation}
\tilde{\theta}_{i}=\theta _{i}-\theta _{id}  \label{theta tilda}
\end{equation}%
where $\theta _{id}$ denotes the desired heading direction, which is to be
specified later. Finally, since the flocking velocity $v_{0}$ is not known
by all agents, $\hat{v}_{fi}\in 
\mathbb{R}
^{2}$ will denote the flocking velocity estimate for agent $i$ and 
\begin{equation}
\tilde{v}_{fi}=\hat{v}_{fi}-v_{0}  \label{vftilda}
\end{equation}%
is the corresponding flocking velocity estimation error.

Before presenting the flocking control scheme, we state a useful lemma.

\begin{lemma}
\label{Lemma stability}\footnote{%
The proof of this lemma is omitted since it is directly based on the proof
of Theorem 1 in \cite{CaiAJC}.}\cite{CaiAJC} Consider the system 
\begin{equation}
\dot{z}=-\alpha R(p)R^{\intercal }(p)z,  \label{zdot lemma}
\end{equation}%
where $\alpha $ is a positive constant and $R(p)$, which was defined in (\ref%
{R}), has full row rank. Given the sets 
\begin{equation}
\begin{array}{l}
\Omega _{1}=\left\{ z:\Lambda (F,F^{\ast })\leq \delta \right\} \\ 
\\ 
\Omega _{2}=\left\{ z:\text{dist}(p,\text{Iso}(F))<\text{dist}(p,\text{Amb}%
(F^{\ast }))\right\}%
\end{array}
\label{Omega12}
\end{equation}%
where $\delta $ is a sufficiently small positive constant and dist$(\cdot )$
denotes the \textquotedblleft distance\textquotedblright\ between a point
and a set, if $z(0)\in \Omega _{1}$, then $z=0$ is an exponentially stable
equilibrium point of (\ref{zdot lemma}). If in addition $z(0)\in \Omega
_{1}\cap \Omega _{2}$, then $F(t)\rightarrow $ Iso$\left( F^{\ast }\right) $
as $t\rightarrow \infty $.
\end{lemma}

The main result of this section is given by the following theorem.

\begin{theorem}
\label{Theorem1} Let $F^{\ast }$ be infinitesimally and minimally rigid, and
the initial conditions for the distance errors satisfy $z(0)\in \Omega
_{1}\cap \Omega _{2}$. Then, the control law 
\begin{eqnarray}
v_{i} &=&\left\Vert u_{i}\right\Vert \cos \tilde{\theta}_{i}  \label{vi} \\
\omega _{i} &=&-c_{i}\tilde{\theta}_{i}+\dot{\theta}_{id}  \label{wi} \\
u_{i} &=&\left[ 
		\begin{array}{c}
			u_{ix} \\ 
			u_{iy}%
		\end{array}%
		\right]=-k_{a}\sum\limits_{j\in \mathcal{N}%
_{i}(E^{\ast })}\tilde{p}_{ij}z_{ij}+\hat{v}_{fi}  \label{ui} \\
\theta _{id} &=&\left\{ 
\begin{array}{ll}
\text{atan2}(u_{iy},u_{ix}), & \text{if }u_{i}\neq 0 \\ 
0, & \text{if }u_{i}=0,%
\end{array}%
\right.  \label{thetadi}
\end{eqnarray}%
\begin{equation}
\overset{\cdot }{\hat{v}}_{fi}=-\alpha \text{sgn}\left( \sum_{j\in \mathcal{N%
}_{i}(E^{\ast })}(\hat{v}_{fi}-\hat{v}_{fj})\text{ }-b_{i}(\hat{v}%
_{fi}-v_{0})\right)  \label{vf_hat_dot}
\end{equation}%
\newline
where $c_{i},k_{a}>0$ are control gains, 
\begin{equation}
b_{i}=\left\{ 
\begin{array}{ll}
1, & \text{if }i\in V_{0} \\ 
0, & \text{otherwise,}%
\end{array}%
\right.  \label{bi}
\end{equation}%
and $\alpha >\gamma _{0}$ is the observer gain, ensures $(z,\tilde{v}_{fi},%
\tilde{\theta}_{i})=0$ for all $i\in V^{\ast }$ is uniformly asymptotically
stable and that (\ref{control_obj_form}) and (\ref{2nd_Obj_Maintenance})
hold.
\end{theorem}

\begin{proof}
We first decompose (\ref{dynamic1}) as follows 
\begin{eqnarray}
\dot{p}_{i} &=&\left[ 
\begin{array}{c}
v_{i}\cos \theta _{i} \\ 
v_{i}\sin \theta _{i}%
\end{array}%
\right]  \label{p_dot} \\
\dot{\theta}_{i} &=&\omega _{i}.  \label{theta_dot}
\end{eqnarray}%
Based on (\ref{thetadi}), we can express $u_{i}$ in polar form: 
\begin{equation}
u_{ix}=\left\Vert u_{i}\right\Vert \cos \theta _{id}\quad \text{and}\quad
u_{iy}=\left\Vert u_{i}\right\Vert \sin \theta _{id}.  \label{polar}
\end{equation}%
Substituting (\ref{vi}) and (\ref{polar}) into (\ref{p_dot}) yields 
\begin{equation}
\dot{p}_{i}=\left[ 
\begin{array}{c}
\left\Vert u_{i}\right\Vert {{\cos \tilde{\theta}_{i}\cos (\tilde{\theta}%
_{i}+\theta _{id})}} \\ 
\left\Vert u_{i}\right\Vert \cos \tilde{\theta}_{i}\sin (\tilde{\theta}%
_{i}+\theta _{id})%
\end{array}%
\right]  \label{pdot1}
\end{equation}%
where (\ref{theta tilda}) was used. After using (\ref{polar}) in ( \ref%
{pdot1}), we obtain 
\begin{equation}
\dot{p}_{i}=B(\tilde{\theta}_{i})u_{i}.  \label{pdot2}
\end{equation}%
where 
\begin{equation}
B(\tilde{\theta}_{i})=\left[ 
\begin{array}{cc}
\cos ^{2}\tilde{\theta}_{i} & -\frac{1}{2}\sin 2\tilde{\theta}_{i} \\ 
\frac{1}{2}\sin 2\tilde{\theta}_{i} & \cos ^{2}\tilde{\theta}_{i}%
\end{array}%
\right] .  \label{Gama_i}
\end{equation}

Now, taking the time derivative of (\ref{zij}) gives 
\begin{equation}
\dot{z}_{ij}=\frac{d}{dt}\left( {p}_{ij}^{\intercal }{p}_{ij}\right) =2{p}%
_{ij}^{\intercal }\left[ B(\tilde{\theta}_{i})u_{i}-B(\tilde{\theta}%
_{j})u_{j}\right] ,  \label{zijdot}
\end{equation}%
which can be rewritten in the following vector form 
\begin{equation}
\dot{z}=2R(p)\mathbf{B}(\tilde{\theta})u  \label{zdot Ch2}
\end{equation}%
where (\ref{R}) was used, $\mathbf{B}(\tilde{\theta})=$ diag$\left( B(\tilde{%
\theta}_{1}),...,B(\tilde{\theta}_{n})\right) \in 
\mathbb{R}
^{2n\times 2n}$, $\tilde{\theta}=\left[ \tilde{\theta}_{1},\ldots ,\tilde{%
\theta}_{n}\right] \in 
\mathbb{R}
^{n}$, and $u=\left[ u_{1},\ldots ,u_{n}\right] \in 
\mathbb{R}
^{2n}$. Likewise, (\ref{ui}) can be rewritten as 
\begin{equation}
u=-k_{a}R^{\intercal }(p)z+\hat{v}_{f}  \label{u}
\end{equation}%
where $\hat{v}_{f}=[\hat{v}_{f1},...,\hat{v}_{fn}]$ $\in 
\mathbb{R}
^{2n}$. If $\tilde{v}_{f}=[\tilde{v}_{f1},...,\tilde{v}_{fn}]$ $\in 
\mathbb{R}
^{2n}$, then from (\ref{vftilda}), we have 
\begin{equation}
\tilde{v}_{f}=\hat{v}_{f}-1_{n}\otimes v_{0}.  \label{v_tilda}
\end{equation}%
After substituting (\ref{u}) into (\ref{zdot Ch2}), we get the closed-loop
system 
\begin{equation}
\dot{z}=-2k_{a}R\mathbf{B}(\tilde{\theta})R^{\intercal }z+2R\mathbf{B}(%
\tilde{\theta})(\tilde{v}_{f}+1_{n}\otimes v_{0})  \label{zdot_3}
\end{equation}%
where (\ref{v_tilda}) was used.

Now, we turn our attention to the flocking velocity estimator. As part of
this proof, we will show that (\ref{vf_hat_dot}) guarantees $\tilde{v}%
_{f}(t)\rightarrow 0$ as $t\rightarrow \infty $. First, notice that 
\begin{equation*}
\sum_{j\in N_{i}(E^{\ast })}(\hat{v}_{fi}-\hat{v}_{fj})=\sum_{j=1}^{n}a_{ij}(%
\hat{v}_{fi}-\hat{v}_{fj}),
\end{equation*}%
where $a_{ij}$ are the elements of the adjacency matrix $A$. Taking the time
derivative of (\ref{v_tilda}) and substituting (\ref{vf_hat_dot}) gives 
\begin{eqnarray}
\overset{\cdot }{\tilde{v}}_{f} &=&-\alpha \text{sgn}\left( (L\otimes I_{2})%
\tilde{v}_{f}+(\mathcal{B}\otimes I_{2})\tilde{v}_{f}\right) -1_{n}\otimes 
\dot{v}_{0}  \notag \\
&=&-\alpha \text{sgn}\left( (\mathcal{M}\otimes I_{2})\tilde{v}_{f}\right)
-1_{n}\otimes \dot{v}_{0}  \label{v_tilda_dot}
\end{eqnarray}%
where we used the fact that $\hat{v}_{fi}-\hat{v}_{fj}=\tilde{v}_{fi}-\tilde{%
v}_{fj}$, $\mathcal{B}:=$ diag$(b_{1},...,b_{n})$, $L$ is the Laplacian
matrix, and $\mathcal{M}:=L+\mathcal{B}$. Since the graph of a rigid
framework is always connected, we know that $G^{\ast }$ is connected.
Therefore, we know from Lemma 3 of \cite{Hong} that $\mathcal{M}$ is
positive definite. Since (\ref{v_tilda_dot}) has a discontinuous right-hand
side, its solution needs to be studied using nonsmooth analysis. Given that $%
sgn(\cdot )$ is Lebesgue measurable and essentially locally bounded, one can
show the existence of generalized solutions by embedding the differential
equation into the differential inclusion \cite{Shevitz} 
\begin{equation}
\overset{\cdot }{\tilde{v}}_{f}\in K\left[ f\right] (\tilde{v}_{f},t)
\label{Diff Incl}
\end{equation}%
where $K\left[ \cdot \right] $ is a nonempty, compact, convex, upper
semicontinuous set-valued map and $f(\tilde{v}_{f},t)=-\alpha $sgn$\left(
\left( M\otimes I_{2}\right) \tilde{v}_{f}\right) -\mathbf{1}_{n}\otimes 
\dot{v}_{0}$.

Consider the Lyapunov function candidate 
\begin{equation}
W=\frac{1}{2}\tilde{v}_{f}^{\intercal }(M\otimes I_{2})\tilde{v}_{f}.
\label{W2}
\end{equation}%
Differentiating $W$ along (\ref{Diff Incl}) yields \cite{Shevitz} 
\begin{equation}
\begin{array}{ll}
\dot{W} & \overset{a.e.}{\in }\dfrac{\partial W_{f}}{\partial \tilde{v}}K%
\left[ f\right] (\tilde{v}_{f},t) \\ 
& \subset \tilde{v}_{f}^{\top }\left( \mathcal{M}\otimes I_{2}\right) \left[
-\alpha \text{sgn}\left( \left( \mathcal{M}\otimes I_{2}\right) \tilde{v}%
_{f}\right) -\mathbf{1}_{n}\otimes \dot{v}_{0}\right] 
\end{array}
\label{Wfdot1}
\end{equation}%
where \textit{a.e.} means \textquotedblleft almost
everywhere\textquotedblright . If we define SGN$\left( x\right) :=\left[
SGN\left( x_{1}\right) ,\ldots ,SGN\left( x_{n}\right) \right] $, $\forall
x\in \mathbb{R}^{n}$ where 
\begin{equation}
SGN\left( x_{i}\right) =\left\{ 
\begin{array}{ll}
1 & \text{for }x_{i}>0 \\ 
\left[ -1,1\right]  & \text{for }x_{i}=0 \\ 
-1 & \text{for }x_{i}<0,%
\end{array}%
\right.   \label{SGN}
\end{equation}%
then (\ref{Wfdot1}) becomes \cite{Shevitz} 
\begin{align}
\dot{W}& =-\alpha \tilde{v}_{f}^{\top }\left( \mathcal{M}\otimes
I_{2}\right) \text{SGN }\left( \left( \mathcal{M}\otimes I_{2}\right) \tilde{%
v}_{f}\right)   \notag \\
& -\tilde{v}_{f}^{\top }\left( \mathcal{M}\otimes I_{2}\right) \left( 
\mathbf{1}_{n}\otimes \dot{v}_{0}\right)   \notag \\
& =-\alpha \left\Vert \left( \mathcal{M}\otimes I_{2}\right) \tilde{v}%
_{f}\right\Vert _{1}-\left( \mathbf{1}_{n}\otimes \dot{v}_{0}\right) ^{\top
}\left( \mathcal{M}\otimes I_{2}\right) \tilde{v}_{f}  \notag \\
& =-\alpha \left\Vert \left( \mathcal{M}\otimes I_{2}\right) \tilde{v}%
_{f}\right\Vert _{1}-\dot{v}_{0}^{\top }\sum_{i=1}^{2n}\left[ \left( 
\mathcal{M}\otimes I_{2}\right) \tilde{v}_{f}\right] _{i}  \notag \\
& \leq -\alpha \left\Vert \left( \mathcal{M}\otimes I_{2}\right) \tilde{v}%
_{f}\right\Vert _{1}+\left\Vert \dot{v}_{0}\right\Vert _{\mathcal{L}_{\infty
}}\left\Vert \left( \mathcal{M}\otimes I_{2}\right) \tilde{v}_{f}\right\Vert
_{1}  \notag \\
& \leq -\left( \alpha -\gamma _{0}\right) \left\Vert \left( \mathcal{M}%
\otimes I_{2}\right) \tilde{v}_{f}\right\Vert _{1}  \label{Wfdot2}
\end{align}%
where $\left\Vert \cdot \right\Vert _{1}$ is the vector 1-norm. For $\alpha
>\gamma _{0}$, $\dot{W}$ is negative definite and therefore $\tilde{v}_{f}=0$
is uniformly asymptotically stable \cite{Shevitz}.\ 

Next, after taking the derivative of (\ref{theta tilda}) and substituting (%
\ref{theta_dot}) and (\ref{wi}), we obtain 
\begin{equation}
\overset{\cdot }{\tilde{\theta}}_{i}=-c_{i}\tilde{\theta}_{i},
\label{thetatilda_dot}
\end{equation}%
which indicates that $\tilde{\theta}_{i}=0$, $\forall i\in V^{\ast }$ is
exponentially stable.

Our overall closed-loop system is composed of three interconnected
subsystems---(\ref{zdot_3}), (\ref{v_tilda_dot}), and (\ref{thetatilda_dot}%
)---which are in the form of (\ref{interconn}) with $y=[\tilde{v}_{f},\tilde{%
\theta}]$. First, notice that (\ref{zdot_3}) with $\tilde{v}_{f}=\tilde{%
\theta}=0$ is the same as (\ref{zdot lemma}) upon application of Lemma \ref%
{LemmaR1n}. As a result, $z=0$\ is exponentially stable for $z(0)\in \Omega
_{1}$ by Lemma \ref{Lemma stability} and therefore, (\ref{zdot_3}) is ISS
with respect to $[\tilde{v}_{f},\tilde{\theta}]$ by Theorem \ref{Khalil_Thm}%
. We can now invoke Theorem \ref{Marquez} to claim that $[z,\tilde{v}_{f},%
\tilde{\theta}]=0$ is a uniformly asymptotically stable equilibrium point of
the interconnected system. If we choose $z(0)\in \Omega _{1}\cap \Omega _{2}$%
, then we know $F(t)\rightarrow $ \ Iso$\left( F^{\ast }\right) $\ as $%
t\rightarrow \infty $ from Lemma \ref{Lemma stability}.

Finally, since $z(t)$ is bounded, we know from (\ref{zij}) that $\tilde{p}%
_{ij}$, $\left( i,j\right) \in E^{\ast }$ is bounded. Therefore, since $%
z(t)\rightarrow 0$ as $t\rightarrow \infty $, we know from (\ref{ui}) and ( %
\ref{pdot2}) that $\dot{p}_{i}(t)-v_{0}(t)\rightarrow 0$ as $t\rightarrow
\infty $ for $\forall i\in V^{\ast }$.
\end{proof}

\begin{remark}
Since atan2$\left( 0,0\right) $ is not defined, we used the definition in (%
\ref{thetadi}) for the desired heading angle to account for the case when $%
u_{i}=0$. Note that the form of the control inputs (\ref{vi}) and (\ref{wi})
is the same irrespective of $u_{i}$. When $u_{i}=0$, (\ref{thetatilda_dot})
still holds while $\dot{p}_{i}=0$ from (\ref{pdot2}), indicating that the
motion of the agents remains bounded.
\end{remark}

\begin{remark}
The time derivative of (\ref{thetadi}), which is needed in (\ref{wi}), is
given by 
\begin{equation}
\dot{\theta}_{id}=\left\{ 
\begin{array}{ll}
\dfrac{u_{i}^{\top }H\dot{u}_{i}}{\left\Vert u_{i}\right\Vert ^{2}}, & \text{%
if }u_{i}\neq 0 \\ 
0, & \text{if }u_{i}=0%
\end{array}%
\right.   \label{Thetadot_d}
\end{equation}%
where%
\begin{equation*}
H=\left[ 
\begin{array}{cc}
0 & 1 \\ 
-1 & 0%
\end{array}%
\right] ,
\end{equation*}%
\begin{eqnarray}
\dot{u}_{i} &=&-k_{a}\sum\limits_{j\in \mathcal{N}_{i}(E^{\ast })}\left(
z_{ij}I_{2}+2{p}_{ij}{p}_{ij}^{\intercal }\right) \left[ B(\tilde{\theta}%
_{i})u_{i}-B(\tilde{\theta}_{j})u_{j}\right]   \notag \\
&&+\overset{\cdot }{\hat{v}}_{fi},  \label{ui dot}
\end{eqnarray}%
and (\ref{pdot2}) and (\ref{zijdot}) were used.
\end{remark}

\begin{remark}
\label{Rem: frame}The flocking control law (\ref{vi})-(\ref{thetadi}) is
implementable in each agent's local coordinate frame. To see this, let $^{j}%
\mathcal{R}_{i}\in SO(2)$ be the rotation matrix representing the
orientation of $\left\{ X_{i},Y_{i}\right\} $ with respect to $\left\{
X_{j},Y_{j}\right\} $, and let a left superscript denote the coordinate
frame in which a variable is expressed. First, since $\tilde{\theta}_{i}$ is
measured from the $X_{i}$ axis, then (\ref{theta tilda}) can be calculated
with respect to $\left\{ X_{i},Y_{i}\right\} $. In fact, the calculation of (%
\ref{theta tilda}) does not need $\theta _{i}$ since $^{0}\tilde{\theta}%
_{i}= $ $^{i}\tilde{\theta}_{i}=-^{i}\theta _{id}$ where the calculation of $%
^{i}\theta _{id}$ from (\ref{thetadi}) uses $^{i}u_{i}$. Due to assumptions
A2 and A3, $^{j}\mathcal{R}_{i}$ is known to agent $i$ for all $j\in 
\mathcal{N}_{i}(E^{\ast })$ so (\ref{vf_hat_dot}) can be calculated with
respect to $\left\{ X_{i},Y_{i}\right\} $. The variable $\dot{u}_{i}$ in (%
\ref{ui dot}) is frame invariant (i.e., it has the same form irrespective of
the coordinate frame) since $^{0}\mathcal{R}_{i}^{\intercal }B(\tilde{\theta}%
_{i})\,\,^{0}\mathcal{R}_{i}=B(\tilde{\theta}_{i})$. Moreover, the term $B(%
\tilde{\theta}_{i})u_{i}-B(\tilde{\theta}_{j})u_{j}$ can be computed in $%
\left\{ X_{i},Y_{i}\right\} $ with knowledge of $^{j}\mathcal{R}_{i}$.
Likewise, $\dot{\theta}_{id}$ in (\ref{Thetadot_d}) is frame invariant since 
$^{0}\mathcal{R}_{i}^{\intercal }H\,\,^{0}\mathcal{R}_{i}=H$. The first term
of (\ref{ui}) is the standard formation shape control term which is known to
be frame invariant \cite{Krick}. Since the second term in (\ref{ui}) is the
integral of (\ref{vf_hat_dot}), then $u_{i}$ can be calculated with respect
to $\left\{ X_{i},Y_{i}\right\} $. Finally, (\ref{thetadi}) can be
implemented in $\left\{ X_{i},Y_{i}\right\} $ once $u_{i}$ is specified
relative to $\left\{ X_{i},Y_{i}\right\} $.
\end{remark}

\section{Target Interception Control}

\label{Target}

Let $v_{T}:=\dot{p}_{T}$ and $e_{T}=p_{T}-p_{n}$ be the target interception
error. Since these quantities are unknown to the followers, observers will
be constructed to estimate them. Thus, $\hat{v}_{Ti}$ will denote the target
velocity estimate for agent $i$ and%
\begin{equation}
\tilde{v}_{Ti}=\hat{v}_{Ti}-v_{T}  \label{vTtilda}
\end{equation}%
is the target velocity estimation error. Further, $\hat{e}_{Ti}$ is the
estimate of the target interception error for agent $i$ and 
\begin{equation}
\tilde{e}_{Ti}=\hat{e}_{Ti}-e_{T}  \label{eTtilda}
\end{equation}%
is the corresponding estimation error.

\begin{theorem}
For initial conditions $z(0)\in \Omega _{1}\cap \Omega _{2}$, the control
law composed of (\ref{vi}), (\ref{wi}), (\ref{thetadi}),%
\begin{equation}
u_{i}=\left\{ 
\begin{array}{l}
-k_{a}\sum\limits_{j\in \mathcal{N}_{i}(E^{\ast })}p_{ij}z_{ij}+k_{T}\hat{e}%
_{Ti}+\hat{v}_{Ti}, \\ 
\quad \text{if }i=1,...,n-1 \\ 
k_{T}e_{T}+v_{T},\text{ if }i=n%
\end{array}%
\right.  \label{u_target}
\end{equation}%
\begin{equation}
\overset{\cdot }{\hat{v}}_{Ti}=-\alpha _{1}sgn\left( \sum_{j\in \mathcal{N}%
_{i}(E^{\ast })}(\hat{v}_{Ti}-\hat{v}_{Tj})+b_{i}(\hat{v}_{Tn}-v_{T})\right)
\label{vT_hat_target}
\end{equation}%
\begin{equation}
\overset{\cdot }{\hat{e}}_{Ti}=-\alpha _{2}sgn\left( \sum_{j\in \mathcal{N}%
_{i}(E^{\ast })}(\hat{e}_{Ti}-\hat{e}_{Tj})+b_{i}(\hat{e}_{Tn}-e_{T})\right)
\label{eT_hat_dot}
\end{equation}%
\newline
where $k_{a},k_{T}>0$ are control gains, $\alpha _{1}>\gamma _{T1}$ and $%
\alpha _{2}>\gamma _{T2}$ are observer gains, $\left\Vert \dot{e}%
_{T}(t)\right\Vert _{\mathcal{L}_{\infty }}\leq \gamma _{T2}$, $\hat{v}%
_{Tn}(0)=v_{T}(0)$, and 
\begin{equation}
b_{i}=\left\{ 
\begin{array}{ll}
1, & \text{if }i=n \\ 
0, & \text{otherwise,}%
\end{array}%
\right.  \label{bi TI}
\end{equation}%
renders $[z,\tilde{\theta}_{i},e_{T}]=0$ for all $i\in V^{\ast }$ uniformly
asymptotically stable, and ensures that (\ref{control_obj_form}) and (\ref%
{2nd_obj_target}) are satisfied.
\end{theorem}

\begin{proof}
First, as in the proof of Theorem \ref{Theorem1}, we can show that (\ref{wi}%
) and (\ref{vT_hat_target}) ensure $\tilde{\theta}_{i}=0$, $\forall i\in
V^{\ast }$ is an exponentially stable and $\tilde{v}_{T}=0$ is uniformly
asymptotically stable for $\alpha _{1}>\gamma _{T1}$ where $\tilde{v}_{T}=[%
\tilde{v}_{T1},...,\tilde{v}_{Tn}]$ $\in 
\mathbb{R}
^{2n}$.

The dynamics of the target interception error is given by 
\begin{equation}
\dot{e}_{T}=v_{T}-B(\tilde{\theta}_{n})(v_{T}+k_{T}e_{T})  \label{e_T_dot}
\end{equation}%
upon use of (\ref{pdot2}) and (\ref{u_target}) for $i=n$. Notice that (\ref%
{e_T_dot}) is ISS with respect to input $\tilde{\theta}_{n}$ since the
unforced system is given by $\dot{e}_{T}=-k_{T}e_{T}$. Therefore, by Theorem %
\ref{Marquez}, the interconnection of (\ref{e_T_dot}) and (\ref%
{thetatilda_dot}) has a uniformly asymptotically stable equilibrium at $%
[e_{T},\tilde{\theta}_{n}]=0$.

Since (\ref{eT_hat_dot}) and (\ref{vT_hat_target}) have a similar structure,
the dynamics of the target interception estimation error can be calculated
as 
\begin{equation}
\overset{\cdot }{\tilde{e}}_{T}=-\alpha _{2}\text{sgn}((\mathcal{M}\otimes
I_{2})\tilde{e}_{T})-1_{n}\otimes \dot{e}_{T}  \label{eT_tilda_dot}
\end{equation}%
where (\ref{eTtilda}) was used, and $\tilde{e}_{T}=[\tilde{e}_{T1},\ldots 
\tilde{e}_{Tn}]\in 
\mathbb{R}
^{2n}$. Given that $e_{T}(t),\tilde{\theta}_{n}(t),v_{T}(t)\in \mathcal{L}%
_{\infty }$, we know from (\ref{e_T_dot}) that $\dot{e}_{T}(t)\in \mathcal{L}%
_{\infty }$. Therefore, we know a bounding constant $\gamma _{T2}$ exists
such that $\left\Vert \dot{e}_{T}(t)\right\Vert _{\mathcal{L}_{\infty }}\leq
\gamma _{T2}$. It then follows from (\ref{eT_tilda_dot}) that $\tilde{e}%
_{T}=0$ is uniformly asymptotically stable when $\alpha _{2}>\gamma _{T2}$.

From the asymptotic stability of the equilibrium $\tilde{v}_{T}=0$ and the
initial condition $\hat{v}_{Tn}(0)=v_{T}(0)$, we have that $\hat{v}%
_{Tn}(t)=v_{T}(t)$, $\forall t\geq 0$. With this in mind, we can rewrite (%
\ref{u_target}) in the following stack form 
\begin{equation}
u=-k_{a}R_{0}^{\intercal }(p)z+\tilde{v}_{T}+k_{T}\tilde{e}_{T}+1_{n}\otimes
(v_{T}+k_{T}e_{T})  \label{u_target_stack}
\end{equation}%
where $R_{0}(p)$ is the rigidity matrix with the last two columns (which
correspond to agent $n$) replaced by zeros. Substituting (\ref%
{u_target_stack}) into (\ref{zdot Ch2}) gives 
\begin{equation}
\begin{split}
\dot{z}& =-2k_{a}R(p)\mathbf{B}(\tilde{\theta})R_{0}^{\intercal }(p)z+2R(p)%
\mathbf{B}(\tilde{\theta})[\tilde{v}_{T}+ \\
& \mathbf{1}_{n}\otimes v_{T}+k_{T}(\tilde{e}_{T}+1_{n}\otimes e_{T})].
\end{split}
\label{zdot_target}
\end{equation}%
\newline

Now, consider the interconnection of (\ref{zdot_target}), (\ref%
{thetatilda_dot}), (\ref{v_tilda_dot}), (\ref{e_T_dot}), and (\ref%
{eT_tilda_dot}). We will check if (\ref{zdot_target}) is ISS with respect to
input $(\tilde{\theta},\tilde{v}_{T},e_{T},\tilde{e}_{T})$. The unforced
system is given by 
\begin{equation}
\dot{z}=-2k_{a}R(p)R_{0}^{\intercal }(p)z.  \label{zdot_target_2}
\end{equation}%
It is not difficult to show that $R(p)R_{0}^{\intercal
}(p)=R_{0}(p)R_{0}^{\intercal }(p)$. Moreover, due to the structure of $R(p)$
, setting its last two columns to zero does not affect its rank. Therefore, $%
R_{0}(p)$ also has full row rank and Lemma \ref{Lemma stability} can be
invoked to conclude that $z=0$ is an exponentially stable equilibrium of (%
\ref{zdot_target_2}) when $z(0)\in \Omega _{1}$. It then follows from
Theorem \ref{Marquez} that $[z,\tilde{\theta},\tilde{v}_{T},e_{T},\tilde{e}%
_{T}]=0$ is a uniformly asymptotically stable equilibrium for the
interconnected system. If $z(0)\in \Omega _{1}\cap \Omega _{2}$, then we
know from Lemma \ref{Lemma stability} that $F(t)\rightarrow $\ Iso$\left(
F^{\ast }\right) $ \ as $t\rightarrow \infty $. Due to the manner in which $%
F^{\ast }$ is constructed, this implies that $p_{n}(t)\in $ conv$\{p_{1}(t),$
$\ldots ,$ $p_{n-1}(t)\}$ as $t\rightarrow \infty $. Since we have proven
that $e_{T}(t)\rightarrow 0$ as $t\rightarrow \infty $, then (\ref%
{2nd_obj_target}) holds.
\end{proof}

\begin{remark}
The target interception control law is also implementable in each agent's
local coordinate frame. One can show this by using the same arguments
outlined in Remark \ref{Rem: frame}.
\end{remark}

\section{Experimental Results}

\label{EXP}

The controllers from Sections \ref{KLC} and \ref{Target} were experimentally
tested on the \textit{Robotarium} system \cite{Pickem2}. This is a swarm
robotics testbed located at the Georgia Institute of Technology that uses
the GRITSBot as the mobile robot platform \cite{Pickem1}. The GRITSBot is a
low-cost, wheeled robot equipped with wireless communication, battery, and
processing boards, and has a footprint of approximately $3\times 3$ cm$^{2}$%
. The MATLAB codes used to implement the controllers are available at
github.com/milad-khaledyan/flocking\_target\_intercep\_codes.git. Due to
page limitations, only the flocking control experiment is presented below.
However, a video of the target interception experiment can be seen at
www.youtube.com/watch?v=HscvM7OtLVQ.

The experiment for the flocking controller (\ref{vi})-(\ref{thetadi}) was
conducted with five robots. The desired formation $F^{\ast }$ was set to a
regular pentagon, which was made infinitesimally and minimally rigid by
introducing seven edges such that $E^{\ast
}=\{(1,2),(1,3),(1,4),(1,5),(2,3),(3,4),(4,5)\}$. The desired distances
between all robots were given by $d_{12}=$ $d_{23}=d_{34}=d_{45}=d_{15}=0.1%
\sqrt{2(1-\cos \frac{2\pi }{5})}$ m and $d_{13}=d_{14}=0.1\sqrt{2(1+\cos 
\frac{\pi }{5})}$ m. The formation was required to move as a virtual rigid
body around a circle. To this end, the desired translational maneuvering
velocity was chosen as $v_{0}(t)=\left[ -r\omega _{0}\sin \omega _{0}t\text{
, }r\omega _{0}\cos \omega _{0}t\right] $ m/s where $r=0.15$ m is the radius
for the circular trajectory and $\omega _{0}=0.3$ rad/s. Figure \ref%
{desired_conf} depicts the desired formation and desired maneuver. Only
robot 1 had access to $v_{0}$. \begin{figure}[!t]
\centering
\includegraphics[keepaspectratio,width=3in,height=3in,trim={0.3in 0.8in 0.5in 1.1in},clip]{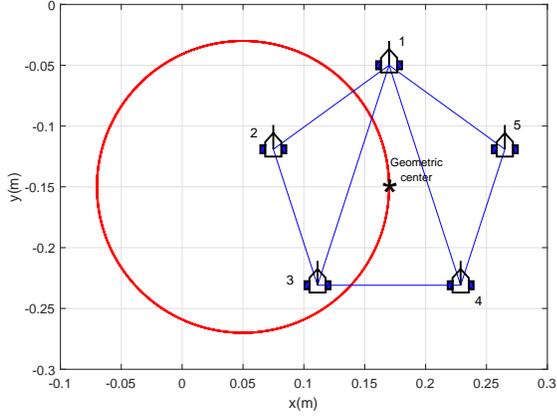}
\caption{Desired pentagon formation along with desired circular trajectory
for the geometric center.}
\label{desired_conf}
\end{figure}

The initial positions and orientations of the robots were randomly selected
while $\hat{v}_{fi}(0)=0$ for $i=1,...,5$. The control gains in (\ref{wi})
and (\ref{ui}) were set to $c_{i}=10$, $i=1,...,5$, and $k_{a}=6$. The
observer gain in (\ref{vf_hat_dot}) was set to $\alpha =0.05$ which
satisfies the constraint $\alpha >\gamma _{0}$ where $\gamma
_{0}=0.045=r\omega _{0}=\left\Vert \dot{v}_{0}(t)\right\Vert _{\mathcal{L}%
_{\infty }}$.

The path of the geometric center of the formation as it maneuvered around
the circle is shown in Figure \ref{paths}. This figure also shows that the
desired formation was successfully acquired from the random initial
configuration. Figure \ref{errors} shows the inter-agent distance errors,
heading angle errors, and flocking velocity estimation errors quickly
converging to approximately zero. The errors are not exactly zero due to
measurement noise and the camera resolution. We can observe from the errors
that the desired formation is acquired after approximately $15$ s. A video
of the experiment can be seen at www.youtube.com/watch?v=nujX1QsVUJI.
\begin{figure}[t]
\centering
\includegraphics[keepaspectratio,width=3in,height=4in, trim={1.5in 0 0.5in
0.5in},clip]{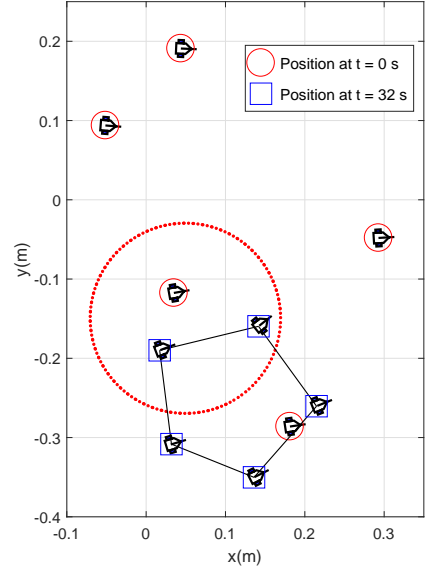}
\caption{Circular maneuver of the geometric center of the formation along
with snapshots of the formation at $t=0$ and $32$ s.}
\label{paths}
\end{figure}
\begin{figure}[t]
\centering
\includegraphics[keepaspectratio,width=3.8in,height=3.076in, trim={0.5in
1.3in 0 0.5in},clip]{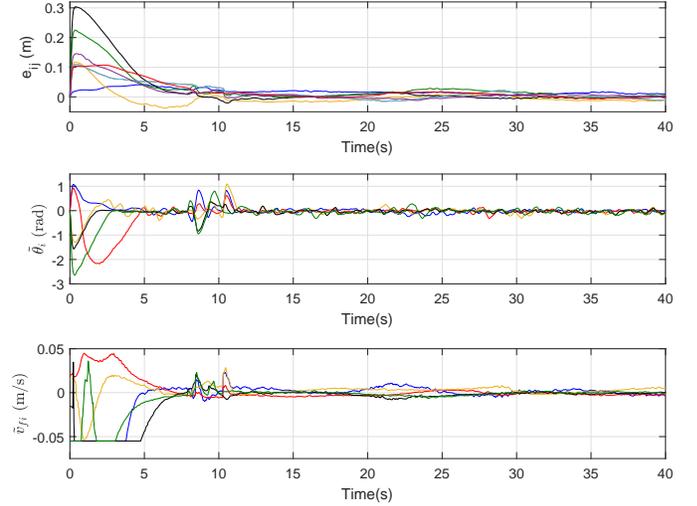}
\caption{Distance errors, $e_{ij}(t)$, $(i,j)\in E^{\ast }$ (top), heading
angle errors, $\tilde{\protect\theta}_{i}(t)$, $i=1,...,5$ (middle), and
flocking velocity estimation errors, $\tilde{v}_{fi}$, $i=1,...,5$ (bottom).}
\label{errors}
\end{figure}

\section{Conclusion}

\label{Conclusion}This paper showed how the distance-based framework can be
applied to nonholonomic kinematic agents to stabilize the inter-robot
distances to desired values while allowing the formation to flock or track
and surround a moving target. The control laws have three main components:
i) an input transformation, ii) the standard gradient descent law for
formation acquisition, and iii) distributed, variable structure observers to
estimate the flocking or target signals not available to certain agents via
their neighbors. The stability analyses showed that the proposed controls
ensure the asymptotic stability of the origin of the error systems.
Experimental results successfully validated the proposed formation control
algorithms.


%


\ifCLASSOPTIONcaptionsoff
\newpage \fi



%

\end{document}